\documentclass[letterpaper]{article} % DO NOT CHANGE THIS
\usepackage{aaai24}  % DO NOT CHANGE THIS
\usepackage{times}  % DO NOT CHANGE THIS
\usepackage{helvet}  % DO NOT CHANGE THIS
\usepackage{courier}  % DO NOT CHANGE THIS
\usepackage[hyphens]{url}  % DO NOT CHANGE THIS
\usepackage{graphicx} % DO NOT CHANGE THIS
\usepackage{natbib}  % DO NOT CHANGE THIS AND DO NOT ADD ANY OPTIONS TO IT
\usepackage{caption} % DO NOT CHANGE THIS AND DO NOT ADD ANY OPTIONS TO IT
\usepackage{algorithm}
\usepackage{algorithmic}

\usepackage{newfloat}
\usepackage{listings}
\DeclareCaptionStyle{ruled}{labelfont=normalfont,labelsep=colon,strut=off} % DO NOT CHANGE THIS
\floatstyle{ruled}
\newfloat{listing}{tb}{lst}{}
\floatname{listing}{Listing}
\usepackage{amsthm, amsfonts, amsmath}
 
\newtheorem{theorem}{Theorem}
\newtheorem{definition}[theorem]{Definition}
\newtheorem{apptheorem}{Theorem}

\newcounter{example}
\newenvironment{example}[1][]{\refstepcounter{example}\par\medskip
   \noindent \textbf{Example~\theexample. #1} \rmfamily}{\medskip}

\newcounter{notation}
\newenvironment{notation}[1][]{\refstepcounter{notation}\par\medskip
   \noindent \textbf{Notation~\theexample. #1} \rmfamily}{\medskip}

\title{Moral Uncertainty and the Problem of Fanaticism}
\author {
    Jazon Szabo\textsuperscript{\rm 1},
    Jose Such\textsuperscript{\rm 1,2},
    Natalia Criado\textsuperscript{\rm 2},
    Sanjay Modgil\textsuperscript{\rm 1}
}
\affiliations {
    \textsuperscript{\rm 1}King's College London, UK\\
    \textsuperscript{\rm 2}VRAIN, Universitat Politecnica de Valencia, Spain\\
    \{jazon.szabo,jose.such,sanjay.modgil\}@kcl.ac.uk, \{jsuch,ncriado\}@upv.es
}

\begin{document}

\maketitle

\begin{abstract}

While there is universal agreement that agents ought to act ethically, there is no agreement as to what constitutes ethical behaviour. To address this problem, recent philosophical approaches to `moral uncertainty' propose aggregation of multiple ethical theories to guide agent behaviour. However, one of the foundational proposals for aggregation --  Maximising Expected Choiceworthiness (MEC) --   has been criticised as being vulnerable to \textit{fanaticism}; the problem of an ethical theory dominating agent behaviour despite low credence (confidence) in said theory. Fanaticism thus undermines the `democratic' motivation for accommodating multiple ethical perspectives. The problem of fanaticism has not yet been mathematically defined. Representing moral uncertainty as an instance of social welfare aggregation, this paper contributes to the field of moral uncertainty by 1) formalising the problem of fanaticism as a property of social welfare functionals and  2) providing non-fanatical alternatives to MEC, i.e. Highest k-trimmed Mean and Highest Median.
\end{abstract}

\section{Introduction}
The recently proposed study of \emph{moral uncertainty} represents a paradigm shift in how philosophers think about ethics \cite{macaskill2020moral}. %TODO more citations?
Instead of aiming at a `one size fits all' approach, moral uncertainty acknowledges that  different ethical perspectives have differing strengths and weaknesses, and that it is rarely the case that there is universal agreement on any given moral issue.  Therefore, under moral uncertainty, an agent aggregates different ethical perspectives so as to yield an overall evaluation. %TODO cite?
In particular such an agent has some \emph{credence} (i.e.  degree of acceptance) in not just one but rather multiple ethical theories. Each theory evaluates the agent's available actions by assigning each action a degree of \emph{choiceworthiness}. A positive choiceworthiness denotes a good outcome while a negative choiceworthiness denotes a bad outcome; the larger the magnitude, the better or worse the outcome of the action, respectively. The agent chooses an action based on both the theories' credences   and the choiceworthiness of the actions.

A fundamental question in moral uncertainty is how to trade off the theories' credences and the actions' choiceworthiness when aggregating the evaluations of different ethical theories. The most influential proposal -- \emph{Maximising Expected Choiceworthiness} (MEC) \cite{macaskill2014normative} --  treats moral uncertainty analogously to empirical uncertainty. Credence corresponds to probability and choiceworthiness corresponds to utility, while MEC itself corresponds to maximising the expected utility.

However, MEC is arguably unsuitable for agent decision making, given its  vulnerability to the \emph{problem of fanaticism} \cite{ross2006rejecting,macaskill2020moral}. That is to say, under MEC, decision making can be dominated by theories  that assign very high stakes to most moral situations (i.e. are evaluated as either extremely desirable or undesirable), even if the agent has low credence in these theories. Hence, the agent's behaviour may be completely determined by only a subset of low-credence theories while completely ignoring other theories. %Hence, some philosophers %such as MacAskill et al. 
%argue that fanaticism `breaks' MEC \cite{macaskill2020moral}.

We concur with  \cite{macaskill2020moral} that fanaticism is unacceptable, since it allows for theories to act as dictators or oligarchs (in the social choice sense of these terms). Thus fanaticism completely undermines democratic motivation for accommodating multiple ethical perspectives. Relatedly, fanaticism subverts any societal endorsement  and trust; a society is unlikely to accept agents which may entirely ignore ethical perspectives that the society has high credence in. %\footnote{The same concern applies to an individual, whose preferences are learnt by an AI agent, but said agent ignores the individual's high credence in a theory, when aggregating the individual's evaluation of an action under a range of theories (recall Footnote \ref{CIRLIndivdualPref})}.
% that have high credence, \sanj{where  the  credence in an ethical theory reflects the extent to which said theory is endorsed by a society.} within that society.

The study of \emph{machine ethics} \cite{moor2006nature} has sought to understand the moral implications of agent behaviour on human stakeholders, so as to design machines that can act ethically, even without human supervision \cite{anderson2005towards}. However, given prevalent significant moral disagreement, both within and between societies \cite{haidt2012righteous}\footnote{\label{foot:cirl}Indeed, the assumption that \textit{individuals} apply a single ethical perspective is questionable \cite{macaskill2020moral}. Hence proposals for value alignment (e.g., \cite{CIRL}) necessitate the learning of individual's ethical preferences.}, it is unclear as to exactly which ethical theory  should guide agents' actions \cite{DBLP:conf/icml/EcoffetL21}. This is especially troubling for machine ethics, where a common methodology has been to implement a single ethical theory.  %TODO cite
As emphasised in \cite{gabriel2020artificial}, given the diversity of ethical perspectives, an agent that acts according to a single ethical theory may not receive societal endorsement, thus jeopardising the project of machine ethics. Indeed, recent work in machine ethics draws on insights from moral uncertainty \cite{DBLP:journals/mima/Bogosian17,bhargava2017autonomous,martinho2021computer,DBLP:conf/aies/DobbeGM20,DBLP:conf/icml/EcoffetL21}, and advocate for the identification of principles for the fair aggregation of various ethical perspectives. Note that, despite the aforementioned issues, all existing work lack a formal definition of fanaticism.
%That is, moral uncertainty in machine ethics aims to guide agent behaviour by finding principles for the fair aggregation of various ethical perspectives upheld by a given society.

Thus, we aim to remedy this limitation by formally defining fanaticism. Moreover, we provide a critique of MEC as the foundational approach to moral uncertainty. While MEC's vulnerability to fanaticism is widely known, no convincing alternatives have been presented.  
We therefore propose non-fanatical methods for resolving moral uncertainty.
In particular, we make the folllowing three \textbf{contributions}:

\textbf{1)} We  formalise  moral uncertainty as social welfare aggregation (as informally proposed by \cite{macaskill2020moral}). Drawing on accepted informal definitions, we formally define fanaticism as a property of social welfare functionals (\textit{swf}s). We do so by precisely defining what it means for a theory to be held with `low credence' (by giving a graded definition of fanaticism) and what it means for a theory to `dominate' (by defining dominant subsets).

\textbf{2)} We define novel, weighted versions of the \textit{swf}s k-trimmed Highest Mean and Highest Median.
    
\textbf{3)} We prove that MEC is in fact maximally fanatical (`Pascalian' fanatical). We prove that neither novel weighted \textit{swf}s are Pascalian and that Highest Median is not fanatical at all. We thereby argue in favour of replacing MEC with one of the proposed novel weighted \textit{swf}s.

The paper is \textbf{structured} as follows. Section \ref{sec:example} presents a running example that illustrates the fanaticism of MEC and introduces concepts and intuitions that will be referenced in later sections. Section \ref{sec:bg} briefly recapitulates background on social welfare aggregation. Moral uncertainty is then defined in social welfare terms, and MEC is defined as an instance of social welfare aggregation
in Section \ref{sec:moral-uncertainty}. We also formalise three weighted alternatives to MEC: Maximin, k-trimmed Highest Mean and Highest Median. We subsequently provide a formal definition of fanaticism  in terms of  \textit{swf}s in Section \ref{sec:fanaticism:def}.   Then Section \ref{sec:fanaticism:thems} states this paper's key results: relating to the extent to which the aforementioned \textit{swf}s are vulnerable to fanaticism.    %TODO conclusion section?

\section{Running example}\label{sec:example}
We present a running example scenario to illustrate  the intuitions underlying our formalism. Consider a small mobile \emph{firefighter robot} FROBO assisting a fire brigade. FROBO's objective is to contain fires in hard-to-access rooms in order to give human firefighters more time to reach these rooms. FROBO has a moral obligation to save lives when it can.

However,  different ethical perspectives  imply different choices of action in order to comply with this  obligation. For example, a version of \emph{utilitarianism} \cite{sep-consequentialism}  interprets the obligation to save lives in terms of maximising the expected number of lives saved. Whereas according to a version of \emph{deontology} \cite{sep-ethics-deontological}, the  obligation to save lives  is interpreted as an absolute obligation.  That is to say, when encountering a person who is in immediate danger of dying, then (\textit{ceteris paribus}\footnote{E.g., if another action option implies saving 2 lives in immediate danger rather than 1 life in immediate danger, then the obligation to save the two would take priority.}) FROBO has an absolute obligation to save that person \cite{kernohan2021descriptive,tarsney2018moral}. In concrete situations these different interpretations may contradict one another.

For example, suppose that after climbing through rubble in a burning residential building, FROBO arrives in a smoke filled hallway. FROBO knows that firefighters will be able to clear the rubble and enter the hallway in about 5 minutes. Until then FROBO is on its own. The hallway has  two doors; one on the right and one on the left. The \emph{right} hand door is open and leads to a burning room containing a collapsed person. FROBO estimates that this room will be burnt out in less than 5 minutes (i.e. before the firefighters arrive), unless FROBO controls the flames with its built-in extinguisher. On the other hand, the door to the \emph{left} is closed. FROBO knows that 4 people are listed as residents of this room. Furthermore, FROBO estimates that it can break down this door just in time for the firefighters arrive. However, the agent doesn't know whether the residents are still home or have managed to escape. Based on the available information, the agent estimates that there is a 50\% chance that the residents are still in the room. Unfortunately, FROBO only has time to exclusively attend either to  the left or right room.

While both of FROBO's theories are in agreement that neither options yield an overall increase in `the good', they disagree on which action is less bad (i.e.`better'). According to FROBO's utilitarian calculations, the agent should choose the left room: the agent can save an expected 2 people, which is better than the 1 person in the right room. According to FROBO's deontological imperative, FROBO should choose the right room given the immediately apparent prospect of the right room  occupant's death, and the possibility  that no one is in the left room.\footnote{ Utilitarianism and Deontology cover a large family of ethical varriants see \cite{sep-consequentialism} and  \cite{sep-ethics-deontological} respectively; FROBO's versions are particular instances.}
 
%, as it is possible that no one is in the left room.}

The exact numerical valuations in Table \ref{table:example}  are based on evaluations assigned in similar moral scenarios (e.g. in \cite{macaskill2020moral}). Importantly, the two theories' evaluations starkly differ in their order of magnitude. FROBO's implementation of deontology posits much higher stakes because
the evaluations of deontology are simply a numerical proxy to the normative force of the obligations they represent\footnote{This representation is inspired by \emph{threshold deontology} \cite{sep-ethics-deontological}, according to which while ethical obligations are not absolute, they have a strong normative force.}.
%unlike utilitarianism, deontology often attaches extreme importance to the adherence of norms \cite{sep-ethics-deontological}.

Indeed, real life examples can be similarly or even more extreme. According to international law,  the prohibition on torture cannot be violated no matter the consequences, e.g. a nation state cannot sanction torture even if it is the only way to prevent existentially catastrophic consequences \cite{assembly1948universal}. It is also worth noting that utilitarianism can also lead to extreme evaluations, such as those advocated by \emph{longtermism} \cite{macaskill2022we}. Thus the theories' different orders of magnitude cannot be simply `normalised away'.

\begin{table}
\centering
\begin{tabular}{llll}
\hline
 & \textbf{Utilitarianism} & \textbf{Deontology} \\ \hline
\textbf{Left room} & -1 & -10000 \\
\textbf{Right room} & -2 & -1000 \\ \hline
\end{tabular}
\caption{Evaluations FROBO's ethical theories}
\label{table:example}
\end{table}

To demonstrate the problem of fanaticism, consider that deontology's credence is very low (e.g. because the vast majority of FROBO's stakeholders advocate utilitarianism). Hence, FROBO assigns only $0.01$ credence to deontology and $0.99$ credence to utilitarianism. However, the deontological evaluations can dominate FROBO's behaviour. Using MEC, FROBO calculates the expected choiceworthiness  -- i.e. the credence-weighted sum of the choiceworthiness --  and picks the action that maximises the \textit{expected choiceworthiness} (\textit{e-cw}). In particular, the   \textit{e-cw} of the left room is $0.99\times(-1)+0.01\times(-10000)=-100.99$, and the  \textit{e-cw} of the right room is $0.99\times(-2)+0.01\times(-1000)=-11.98$. In both cases, the deontological theory
dominates, and so under MEC, FROBO chooses the right room. 

The above  illustrates the \emph{problem of fanaticism}.  FROBO's behaviour is solely dictated by the deontological theory  in which FROBO has very small credence, but which dominates the expected choiceworthiness calculations.  This is highly undesirable.  FROBO ignores utilitarianism despite  this theory being advocated by the vast majority of FROBO's stakeholders. Hence avoiding fanaticism is important, if FROBO is to obtain societal approval.
\section{Background} \label{sec:bg}
We introduce social welfare aggregation \cite{sep-social-choice}, by way of background for our definitions in  later sections. %TODO redo this intro?

%\subsection{Social Welfare Aggregation} \label{sec:bg-basics}
Let $N = \{1, ..., n\}$ be the set of \emph{individuals} (or voters) and $X = \{x, y, z, ...\}$  the set of \emph{social alternatives}. Each individual $i \in N$ has a \emph{welfare function} $u_i: X \to \mathbb{R}$, so $u_i(x)$ represents the welfare of individual $i$ under alternative $x$. A list of welfare functions for each individual $P = \langle u_1, ... u_n \rangle$ is called a \emph{profile}.

Let $D$ denote the domain of all possible profiles and $\mathbf{R}_X$ the set of possible total orders on the set of social alternatives $X$. Then, a \emph{social welfare functional} (\textit{swf}) $f: D \to \mathbf{R}_X$ maps each welfare profile to a total order on the set of social alternatives.

%\subsection{Measurability and interpersonal comparability}
In social welfare aggregation, one can make different assumptions about how much information is encoded by the \textit{swf}s, via use of \emph{meaningful statements} \cite{sep-social-choice}. 
\begin{itemize}
    \item \emph{Level comparison}: Individual $j$'s welfare under alternative $y$ is at least as great as individual $i$'s welfare under alternative $x$; formally $u_i(x) \le u_j(y)$.
    \item \emph{Unit comparison}: We can divide $\delta_i$ by $\delta_j$, where $\delta_i$ is the number equal to individual $i$'s welfare gain or loss when switching from alternative $y_1$ to alternative $x_1$  and  $\delta_j$ is individual $j$'s welfare gain or loss when switching from alternative $y_2$ to alternative $x_2$; formally $\delta_i / \delta_j=(u_i(x_1)-u_i(y_1))/(u_j(x_2)-u_j(y_2))=w$, where $x_1, x_2, y_1, y_2 \in A$ and $w \in \mathbb{R}$.
    \item \emph{Zero comparison}: Individual $i$'s welfare under alternative $x$ is greater than, equal to or less than zero; formally $\mathit{sign}(u_i(x))=w$, where $w \in \{-1, 0, 1\}$ and $\mathit{sign}$ is a function that maps negative numbers to $-1$, zero to $0$, and positive numbers to $+1$.
\end{itemize}

In the above definitions a comparison is said to be \emph{intrapersonal} if $i = j$ and \emph{interpersonal} if $i \ne j$. In this paper we use the \emph{ratio-scale measurability with full interpersonal comparability} (RFC) assumption, i.e. that intra- and interpersonal comparisons of all three kinds (level, unit, and zero) are meaningful. Formally, RFC means  that two profiles $P = \langle u_1, u_2,..., u_n\rangle$ and $P' = \langle u'_1, u'_2,..., u'_n\rangle$ contain the same information if, for each individual $i \in N$, $u'_i=au_i$, where $a$ is the same positive real number for all individuals ($a \in \mathbb{R}^+$). Informally, RFC means that the different welfare functionals are assumed to be normalised to the same numeric scale and as such further normalisation is impossible.
\section{Moral uncertainty} \label{sec:moral-uncertainty}
We now formalise moral uncertainty  in terms of social choice, based on assumptions underpinning MEC.
Firstly, note that MEC assumes that ethical theories are on the same numerical scale, i.e. they are ratio-scale and that the evaluations of ethical theories can  be meaningfully compared across theories; MEC makes the RFC assumption\footnote{Moral uncertainty literature is yet to examine decision making under stronger assumptions than RFC.} \cite{macaskill2020moral}.  Since we are interested in  providing viable alternatives to MEC,  the formalisms presented in this paper also assume RFC. We therefore formalise ethical theories as individuals and their evaluations as welfare functions, while formalising credences as the `weights' of individuals. Then, MEC and other methods of resolving moral uncertainty are formalised as \textit{swf}s, and we provide novel alternatives to MEC:  the  weighted k-Trimmed Highest Mean and the weighted Highest Median. Section \ref{sec:fanaticism} then uses these definitions to first define the problem of fanaticism as a property of \textit{swf}s. We then give results that evaluate the extent to which these \textit{swf}s are fanatical.% (in particular, that the Highest Median avoids fanatacism).

\subsection{Ethical theories and ethical frameworks}
 
We now define how one can account for moral uncertainty when evaluating actions, as an instance of social welfare aggregation applied to ethical theories and their credences.

Recall that under moral uncertainty, the agent has multiple \emph{ethical theories}, each of which  provides a real-valued evaluation of the actions available to the agent.  For each theory $t$, the credence function $c$ assigns a measure of the extent (on a scale from $0$ to $1$) to which $t$ is advocated by a given society as being appropriate  for ethical evaluation of actions. Hence, an agent's  ethical decision making under moral uncertainty is defined on the basis of an 
\emph{ethical framework}:
%a tuple  consisting of ethical theories and a credence function.
%that an agent uses for ethical decision making under moral uncertainty:

\begin{definition} \label{def:ethical-framework} [Ethical framework]
An \emph{ethical framework} is a tuple $F = (T, c)$  consisting of a set of ethical theories $T$ and a credence function $c: T \to (0, 1]$. Given a set of actions $A$, each ethical theory $t \in T$ assigns a real-valued evaluation to each action $a \in A$, i.e. $t: A \to \mathbb{R}$.
\end{definition}
%Note that when the set of actions $A$ is clear from the context, we omit the subscript and write $t$ instead of $t_A$.

For simplicity we require that for $F = (T, c)$,  the theories' credences sum up to 1, i.e. $\sum_{t \in T}c(t) = 1$. Furthermore, note that in this work we are agnostic with respect to how the evaluations of ethical theories are elicited.

\begin{example}
In our running example, $A = \{l, r\}$ where  $l$ and $r$ respectively denote  the actions  `enter left room' and `enter right room'.   FROBO's ethical framework is $F_\mathit{FROBO} = (\{d, u\}, c)$, where
the deontological theory ($d$) evaluates $r$  as highly impermissible given the possibility that residents of the left room might die: $d(r) = -1000$. On the other hand,  $l$ is even more impermissible because doing so will guarantee that the occupant of the right room dies: $d(l) = -10000$. By contrast, utilitarianism ($u$)  evaluates $l$ as impermissible ($u(l) = -1$) and $r$ as more impermissible ($u(r) = -2)$ (recall Table \ref{table:example}).
Utilitarianism is  advocated  to an extremely high degree, c.f. deontology: $c(u) = 0.99$ and $c(d) = 0.01$.

\end{example}

\subsection{Evaluation aggregation}
Addressing the problem of moral uncertainty amounts to aggregating individual ethical theories' evaluations so as to rank actions.
%
%using a given ethical %framework to 
% rank actions and so 
%guide agent behaviour. %This amounts to how one
%aggregates the individual theories' evaluations.%   the \emph{evaluation aggregation problem}.
As \cite{macaskill2020moral} point out,  under a social welfare perspective  (recall Section \ref{sec:bg}) 
social alternatives equate with actions and individuals (voters) equate with ethical theories. Likewise, evaluation aggregation methods can be defined through \textit{swf}s. However, ethical theories are weighted by their credence. Therefore, some of the \textit{swf}s considered in this paper also take as input the weights of individuals. Formally:

\begin{definition}\label{def:evaluation-aggr} [Evaluation aggregation]
Let $F = (T, c)$  and $A$ a set of actions.  \emph{Evaluation aggregation} is defined as an instance of social welfare aggregation, where:
\begin{itemize}
    \item the social alternatives are the actions $A$;
    \item the individuals are the theories $T$;
    \item the social welfare of individual $i$ representing a theory $t \in T$ is given by the theory's evaluation function, i.e. $u_i = t$, and;
    \item if the social welfare makes use of a weight function $w$, then the weight of an individual $i$, representing a theory $t \in T$, is given by the theory's credence, i.e. $w(i) = c(t)$, %TODO is the weighing bit correct?
\end{itemize}
\end{definition}

\begin{notation}\label{not:evaluation-aggr} Abusing notation we may write $f(F, A)$ to denote the action ordering   resulting from applying an \textit{swf} $f$ to aggregate evaluations,  given   $F = (T, c)$ and actions $A$.
\end{notation}

\begin{example}
 Given  $F_\mathit{FROBO} = (\{d, u\}, c)$ and 
 $A = \{ l, r \}$, 
Definition \ref{def:evaluation-aggr}  formulates aggregation of the ethical evaluations under moral uncertainty as a social welfare aggregation problem. The individuals are $N = \{i_d, i_u\}$ ($i_d$ and  $i_u$ respectively correspond to deontology and utilitarianism). Deontology's welfare function is $u_{i_d}(l) = -10000$ and $u_{i_d}(r) = -1000$. Utilitarianism's welfare function is $u_{i_u}(l) = -1$ and $u_{i_u}(r) = -2$. The individuals' weight functions  are given by the credence functions, i.e. $w: N \to (0, 1]$ and  $w(i_u) = c(u) = 0.99$ and $w(i_d) = c(d) = 0.01$.
\end{example}

\subsection{Social welfare functionals for evaluation aggregation}
We present four social welfare functionals (\textit{swf}s). The \textit{swf}s MEC and Maximin  have been suggested by the  moral uncertainty literature. In this paper we propose two novel weighted \textit{swf}s --
k-trimmed Highest Mean and Highest Median -- which can be understood as modified, less fanatical versions of MEC. %As we show later, while MEC and Maximin are both fanatical, both these modifications of MEC avoid fanaticism. %TODO is this last sentence clear?
Note that in this section we define these \textit{swf}s in terms of moral uncertainty (see Definition \ref{def:evaluation-aggr}).
%i.e. by referring to ethical theories, not individuals,   to credence functions, not weight functions, etc. (see Definition \ref{def:evaluation-aggr}).

\emph{Maximising Expected Choiceworthiness} (MEC) is a foundational method in moral uncertainty research, and its %TODO cite
 vulnerability to fanaticism has long been informally recognised. We formally prove that MEC is fanatical in Section \ref{sec:fanaticism:thems}. Note that MEC is  more commonly known as \emph{weighted utilitarianism} in the social choice literature \cite{harsanyi1955cardinal, macaskill2020moral}. However, to avoid confusion with the ethical theory utilitarianism, we will call the functional MEC (as it is known in the  moral uncertainty literature), and denote it formally as $\textit{mec}$.

MEC orders actions based on the credence weighted sum of the theories' evaluations. This is in fact equivalent to ordering actions based on their respective weighted arithmetic mean\footnote{This equivalence holds because the credences of the different theories add up to 1.}. By conceptualising MEC in terms of the weighted arithmetic mean, the intuition behind the later definitions of k-trimmed Highest Mean and Highest Median become clearer. %TODO explain more! people that know moral uncertainty might not find this obvious!
Therefore, we define $\textit{mec}$ by reference to the weighted arithmetic mean function $\mathit{wam}$.
%TODO intuition!
Formally given an ethical framework $F = (T, c)$, an action $a$'s weighted arithmetic mean is defined as: $\mathit{wam}(F, a) = \sum_{t \in T}c(t)t(a)$. %TODO note: we don't normalise because the weights add up to 1

\begin{definition}\label{def:mec} [MEC ($\textit{mec}$)]
 Let $a, b \in A$ be actions and let $F = (T, c)$. Then $\textit{mec}(F, A) = \preceq_\textit{mec}$, where  $a \preceq_\textit{mec} b$ iff $\mathit{wam}(F, a) \le \mathit{wam}(F, b)$.
\end{definition}
 
Note that we have calculated the weighted arithmetic means of FROBO's actions in Section \ref{sec:example},

We now consider  the Maximin \textit{swf}, which in the literature is considered as an inferior  alternative to MEC \cite{DBLP:journals/mima/Bogosian17}. This is because Maximin is extremely vulnerable to fanaticism as it disregards the credences of ethical theories.  Maximin orders actions based solely on which  maximises the minimum evaluation of any ethical theory.

\begin{definition} \label{def:maximin} [Maximin ($\mathit{mm}$)]
 Let $a, b \in A$ be actions and let $F = (T, c)$.  Then $\mathit{mm}(F, A) = \preceq_\mathit{mm}$, where $a \preceq_\mathit{mm} b$ iff $\mathit{min}_{t \in T}t(a) \le \mathit{min}_{t \in T}t(b)$. 
\end{definition}

\begin{example} \label{ex:mm}
Consider $F_\mathit{FROBO}$. The minimal evaluation of   $l$ is $\mathit{min}\{-1, -10000\} = -10000$, whereas %(see Table \ref{table:example}). 
 the minimal evaluation of $r$ is $\mathit{min}\{-2, -1000\}  = -1000$. %Since $-10000 < -1000$, the action 
 Hence $r$ has the maximal minimum evaluation and so using Maximin, %$\mathit{mm}$, 
 FROBO will choose to enter the right room.
\end{example}

The next \textit{swf} -- k-trimmed highest mean -- modifies MEC so as to  \textit{some extent}  avoid fanaticism (as  shown in Section \ref{sec:fanaticism}).  The underlying statistical intuition  is that  the arithmetic mean is known to be sensitive to outliers \cite{maronna2006robust}. We believe that the idea of outlier sensitivity%\footnote{More precisely, the notion of a \emph{breakdown point} \cite{maronna2006robust}, which measures the proportion of observations that can be arbitrarily large before the estimator itself becomes arbitrarily large.} 
is related to fanaticism. %TODO say more? move from footnote to body of text?
We therefore modify MEC by making the statistical estimator it uses more robust to outliers. That is, we replace the weighted arithmetic mean with a trimmed weighted arithmetic mean. Trimming means removing some of the most extreme values. %TODO cite?
While unweighted versions of trimmed mean functionals have been defined   \cite{hurley2002combining}, our weighted version is (to our best knowledge) novel. %TODO remove "to our best knowledge"?
We first need some auxiliary definitions. 

If $F = (T, c)$ and $a \in A$, then $\mathit{se}(F, a) = \mathit{sort}(\langle t(a) | t \in T\rangle)$, where $\mathit{sort}$ sorts the elements of a list in a non-descending order. That is, $\mathit{se}$  maps any action $a$ and ethical framework $F$ to a sorted list of $a$'s evaluations by the theories in $T$. Let $\mathit{st}(F, a)$ be the theories corresponding to the sorted evaluations, i.e. $t$ is the $i$th element of $\mathit{st}$ , $\mathit{st}(F, a)_i = t$, iff $t(a)$ is the $i$th element of $\mathit{se}$, $\mathit{se}(F, a)_i = t(a)$. For FROBO $\mathit{se}(F_\mathit{FROBO}, l) = \langle -10000, -1 \rangle$ and $\mathit{st}(F_\mathit{FROBO}, l) = \langle d, u \rangle$ since $d(l) = -10000 < u(l) = -1$.

We now want to  trim the `bottom' $k$ portion of the theories, as weighted by their credences. Hence $\mathit{bottom_k}(a)$ is the set of theories with the lowest evaluations of $a$ such that their total credence is at most $k$. That is, for any $k \in [0, 0.5)$ and $a \in A$: $$\mathit{bottom_k}(a) = \{\mathit{st}(F, a)_i | 1 \le i < kend] \}$$ where $\mathit{kend}$ is such that $\sum_{i \in [1, kend)}c(\mathit{st}(F, a)_i) \le k$ and $\sum_{i \in [1, kend + 1)}c(\mathit{st}(F, a)_i) > k$.

Trimming the `top' $k$ portion of the theories is defined symmetrically. For any $k \in [0, 0.5)$:
$$\mathit{top_k}(a) = \{\mathit{st}(F, a)_i | kstart < i \le n] \} \; where \; n = |T|$$ and $\mathit{kstart}$ is such that $\sum_{i \in (\mathit{kstart}, n]}c(\mathit{st}(F, a)_i) \le k$ and $\sum_{i \in (\mathit{kstart} - 1, n]}c(\mathit{st}(F, a)_i) > k$.

\begin{definition} [k-Trimmed Highest Mean ($k$-$\mathit{thm}$)] \label{def:kthm}
 Let $a, b \in A$   and  $F = (T, c)$. Let $k$ be a real number such that $k \in [0, 0.5)$. Then  $\textit{k-thm}(F, A) = \preceq_\textit{k-thm}$, where $a \preceq_\textit{k-thm} b$ iff $\mathit{wam}(F_a, a) \le \mathit{wam}(F_b, b)$, where\\
\noindent $F_a = (T, c_a)$, $c_a(t) = 0$ for $t \in (\mathit{bottom_k}(a) \bigcup \mathit{top_k}(a))$, else $c_a(t) = c(t)$;\\
\noindent $F_b = (T, c_b)$, $c_b(t) = 0$ for $t \in (\mathit{bottom_k}(b) \bigcup \mathit{top_k}(b))$, else $c_b(t) = c(t)$.
\end{definition}   
Note that in the case $k = 0$, $k$-$\mathit{thm}$ is equivalent to $mec$. 

\begin{example} \label{ex:kthm}
Suppose FROBO uses $0.1$-$\mathit{thm}$ ($k = 0.1$). First consider the sorted evaluation of FROBO's ethical theories regarding the left room: $-10000$ by   $d$ and $-1$ by $u$. Formally, $\mathit{se}(F_\mathit{FROBO}, l) = \langle -10000, -1 \rangle$ and $\mathit{st}(F_\mathit{FROBO}, l) = \{ d, u \}$. Before calculating the weighted arithmetic mean, we see if any theories must be trimmed. Starting at the bottom, $d$ has the lowest value. Note that $c(d) = 0.01 \le k = 0.1$  and so we want to trim $d$ away. At the same time $c(u) = 0.99$ and $c(d) + c(u) = 1 > 0.1$ and so trimming away $u$ would mean trimming away too much of the credence. Therefore, $\mathit{kend} = 2$ and $\mathit{bottom_k}(l) = \{ d\}$. At the top, where $u$ has the highest value, we do not trim away any theories because $c(u) = 0.99 > 0.1$. This means that $\mathit{kstart} = 2$ and $\mathit{top_k}(l) = \{\}$. Therefore, we calculate the trimmed weighted mean with the trimmed credences, i.e. $c'(d) = 0$ (since $d \in \mathit{bottom_k}(l)$) and $c'(u) = c(u) = 0.99$ (since $u \notin \mathit{bottom_k}(l)$ and $u \notin \mathit{top_k}(l)$). Therefore, the trimmed weighted mean only considers utilitarianism's evaluation for the left room, i.e. the $0.1$-trimmed weighted arithmetic mean is $-1$. For similar reasons, $0.1$-$\mathit{thm}$ will disregard deontology for also the right room and so the trimmed mean is $-2$. Therefore, $l \succ_{0.1-thm} r$ because $-1 > -2$ and thence FROBO chooses the left room. In other words, $0.1$-$thm$ enables FROBO to avoid fanaticism in this case.
\end{example}

The final functional is the \textit{highest median}, derived from MEC by maximally trimming the arithmetic mean. This is because the median is the k-trimmed mean in the limit, as $k \to 0.5$, when all but one (if n is odd) or two (if n is even) elements  are trimmed. As shown in Section \ref{sec:fanaticism}, median is `maximally' non-fanatical. Note that while our definition of the weighted highest median is novel, it is based on a well-known (unweighted) \emph{majority judgment} aggregation method  \cite{balinski2007theory, fabre2021tie}.
The weighted median of an action $a$'s evaluation is the evaluation such that at most half the theories (weighted by their credence) have a higher evaluation and at most half the theories (weighted by their credence) have a lower evaluation. %TODO Is this intuition clear?

We first provide some auxiliary definitions. Recall that for any action $a$ and  any $F = (T, c)$ (where $n = |T|$), $\mathit{se}$ maps $a$ and $F$ to a sorted list of the evaluations of $a$ by the theories in $T$, and $\mathit{st}(F, a)$ is a list of the corresponding theories. %TODO is this good enough intuition here?
Let $w_i = c(\mathit{st}(F, a)_i)$ be the credence of the theory with the $i$th lowest evaluation of action $a$. Let   $m \in [1, n]$ be such that:  $$\sum_{i \in [1, m-1]}w_i \le 1/2 \; \; and \; \sum_{i \in [m + 1, n]} w_i \le 1/2$$  That is, $m$ is the index of any theory such that the total credence of the theories with lower/higher evaluations of action $a$ is at most $0.5$. There are two possibilities: either $m$ is uniquely determined or there are two distinct numbers $m_1, m_2$ such that the above holds. %TODO cite
If $m$ is uniquely determined, let $\mathit{wmedian}(F, a) = \mathit{se}(F, a)_m$; otherwise let $\mathit{wmedian}(F, a) = (\mathit{se}(F, a)_{m_1}+ \mathit{se}(F, a)_{m_2}) / 2$.

\begin{definition}\label{def:hm} [Highest Median ($\mathit{hm}$)]
 Let $a, b \in A$ be actions and let $F = (T, c)$. Then $\mathit{hm}(F, A) = \preceq_\mathit{hm}$, where $a \preceq_\mathit{hm} b$ iff $\mathit{wmedian}(F, a) \le \mathit{wmedian}(F, b)$.
\end{definition}

\begin{example} \label{ex:hm}
Recall that  $\mathit{se}(F_\mathit{FROBO}, l) = \langle -10000, -1 \rangle$ and $\mathit{st}(F_\mathit{FROBO}, l) = \langle d, u \rangle$ given $d(l) = -10000 <~u(l)$
$= -1$. The median evaluation is such that at most half the credence weighted evaluations may be lower and at most half the credence weighted evaluations may be higher. $-10000$ cannot be the median as $u$'s evaluation  ($-1$) is higher than $d$'s evaluation ($-10000$), and $c(u) = 0.99 > 0.5$. No theory has a higher evaluation  than $u$'s evaluation ($-1$)  %(and so  the credence of the theories' with higher evaluations is $0$)
and while $d$'s evaluation ($-10000$) is lower,  $c(d) = 0.01$ is less than half. Therefore $m = 2$  and the median evaluation $\mathit{wmedian}(F_\mathit{FROBO}, l) = -1$. For similar reasons, $\mathit{wmedian}(F_\mathit{FROBO}, r) = -2$; fanaticism is thus avoided.\\[-20pt]
\end{example}
\section{Fanaticism} \label{sec:fanaticism}
We now formalise the notion of fanaticism, i.e. what it means for a low-credence theory to dominate. Our proposed definition is not binary, but  rather   graded in that it ranges from not at all fanatical to fanatical in an extreme sense (i.e. `Pascalian'). This graded definition yields insights as to how vulnerable different  functionals are to fanaticism. We will show that both MEC and Maximin are Pascalian, while k-trimmed Highest Mean to some extent  avoids fanaticism, and Highest Median completely avoid fanaticism. %TODO Include? "Therefore, we argue that MEC should be replaced by a non-fanatical functional

\subsection{Defining Fanaticism} \label{sec:fanaticism:def}
 %The question  arises as to the differential roles  that an agent's  ethical theory evaluations  and their credences   should play in decision making. 
For fanatical theories, the relative magnitude of evaluations can be so extreme that the credences are essentially ignored (e.g., $d(l) = -10000$ and $d(r) = -1000$ dominating FROBO's decision making, despite  $c(d) = 0.01$). Fanaticism is a kind of \emph{oversensitivity} to the evaluations of ethical theories and an \emph{undersentivity} to their credences \cite{newberry2021parliamentary}. In the case of MEC, it is well known that this oversensitivity is due to the larger evaluative stakes posited by the theories \cite{macaskill2020moral}. %TODO cite more?

%On the other hand,   Maximin  chooses the action which has the highest minimum evaluation by any theory. 
Maximin is also known to be oversensitive to evaluations \cite{DBLP:journals/mima/Bogosian17}; choosing the action with maximal minimum evaluation can ignore credences in favour of evaluations. We, therefore, consider Maximin fanatical. Unlike MEC, Maximin is \textit{not} sensitive to \textit{any} high-stakes theory. Rather, Maximin is sensitive to high-stakes `pessimistic' theories, e.g. ones that give large negative evaluations. %The following example shows that Maximin is insensitive to high stakes, so long as they are positive.

%\begin{example}
To see why, assume that FROBO has an alternative `optimistic' deontological theory $o$ that evaluates actions positively, i.e. $o(l) = 1000$ and $o(r) = 2000$. Then maximin would choose entering the left room, as the minimum evaluation in either case is the utilitarian evaluation: $u(l) = -1$, $u(r) = -2$. That is, Maximin is insensitive to $o$.    
%\end{example}

The observation that Maximin is fanatical argues for the view that fanaticism arises not only because of high stakes, but due to other features of the theories' evaluations. In the case of Maximin, these potentially include the sign of the evaluations. We hence understand fanaticism as a property of \textit{swf}s, where some feature(s)\footnote{In this work we are agnostic as to what these features may be.} of a subset of theories dominate the agent's decision making.
Thus, we formalise the idea that fanaticism allows some theories to completely dominate the resulting ordering of actions, by defining the notion of a \emph{dominant subset} %\footnote{Note that the notion of a dominant subset is distinct from the Arrovian notion of a \emph{decisive subset} \cite{kenneth1963arrow}.}
of theories: one that has a final say in the overall ordering of actions irrespective of what the other theories are. More precisely, an ethical framework with a dominant subset of theories leads to the same ordering of actions as that obtained by removing the non-dominant theories. We first define what it means for a framework to be restricted to just a subset of theories: 

\begin{definition} \label{def:restricted} [Restricted framework]
Given $F = (T, c)$ and a subset of theories $T' \subset T$, then $F' = (T', c')$ is obtained by \emph{restricting} $F$ to $T'$ if for any theory $t \in T'$: $c'(t) = n \times c(t)$ where $n$ is a normalising constant $n = \frac{1}{\sum_{t \in T'}c(t)}$.
\end{definition}

Note that we normalise the credence function so that the different credences add up to 1. %TODO say that otherwise proportional?

\begin{definition}\label{def:dominant-subset} [Dominant subset]
Let $F=(T, c)$, $A$ a set of actions,  and $f$ a social welfare functional. Then  $T_d\subset T$ is said to be a \emph{dominant subset} of theories if   $f(F, A) = f(F_d, A)$ and $f(F_d, A) \ne f(F_y, A)$, where
\begin{itemize}
    \item $F_d$ is obtained by restricting $F$ to $T_d$ and
    \item $F_y$ is obtained by restricting $F$ to $T_y$ where $T_y = T \setminus T_d$.
\end{itemize}
\end{definition}

That is,  $T_d$ determines the order of actions in $f(T, A)$, i.e. $f(T, A) = f(T_d, A)$. Moreover, they do so in spite of the differing evaluations of the non-dominant (`yielding') theories, i.e. $f(T_d, A) \ne f(T_y, A)$.
%\begin{example}
In our running example, when applying either MEC or Maximin, $T_d = \{d\}$ is a dominant subset (where $T_y = \{u\}$)   because both $l \prec_\mathit{mec} r$ and $l \prec_\mathit{mm} r$, which are the same as that of deontology $l \prec_d r$.
%\end{example}

Note that fanaticism does not amount to the mere possibility of dominant subsets, but is rather a \textit{systematic} vulnerability to dominant subsets. To see why mere possibility doesn't constitute fanaticism, consider the following variation of our running example. Suppose that in   $F_{FROBO}$, we substitute a non-fanatical deontology $d'$ for $d$, whereby  ${d'}(l) = -2$ and ${d'}(r) = -1$. Suppose  $c(u) = c(d') = 0.5$; the theories have equal credence and give symmetric but opposite evaluations. Then FROBO has to randomly choose between $l$ and $r$. However, if $F_{FROBO}$ included a third theory $t$, this could be used as a tie-breaker, even if FROBO has very little credence ($c(t) = 0.01$) in $t$. Suppose $t(l) = -1$ and $t(r) = 0$. Then $F_{FROBO}$ contains $u$, $d'$ and $t$, and a reasonable aggregate choice would be entering the right room, thus making $\{t\}$ a dominant subset.  However, we suggest that this is \textit{not} a case of fanaticism; rather, it is entirely reasonable that $t$ serves as a tiebreaker.

Therefore, fanaticism is  \textit{systematic} in the sense that dominant subsets are always possible regardless of the actions or the non-dominant subset of a framework's ethical theories. This aligns with the idea that fanaticism is a systematic oversensitivity to the evaluations of ethical theories \cite{newberry2021parliamentary}. %TODO cite more
In particular, fanaticism means that given an arbitrary framework, it is always possible to extend the framework with additional low-credence theories such that these low-credence theories are a dominant subset.

We first define what it means to extend a framework to include additional theories.

\begin{definition}\label{def:extended} [Extended framework]
Given $F = (T, c)$ and some  $T' \supset T$, then $F' = (T', c')$ is   obtained by \emph{extending} $F$ with $T'$ if for any $t \in T$:   $c'(t) = n \times c(t)$ where $n$ is a normalising constant $n = \frac{1 - \sum_{t \in (T' \setminus T)}c'(t)}{\sum_{t \in T} c(t)}$.
\end{definition}

Note that there are multiple different ways of extending $F$ with theories $T'$, each individuated by distinct credence functions $c'$ for the theories in $T' \setminus T$. %TODO clear?
Also note that the normalising constant ensures that $c'$ sums  to 1.

Fanaticism arises due to theories with low credence. However, `low credence' is a vague term. While low credence intuitively means credence less than $0.5$, the exact cut-off point between low and non-low credence is not clear. We, therefore, give a \textit{graded} definition of fanaticism that refers to $k$-fanaticism (for some $k \in (0, 0.5)$) and where $k$ is (an upper bound on) the credence of dominant theories.

What constitutes `problematic' $k$-fanaticism depends on our notion of what `low credence' means. If the cut-off point for low credence is say  $0.3$, then an $f$ that is $0.3$-fanatical (but not $k$-fanatical for any $k < 0.3$) is `problematically' fanatical, whereas an  $f'$ that is $0.4$-fanatical (but not $k$-fanatical for any $k < 0.4$) is not `problematically' fanatical. We remain agnostic as to where this cut-off point for low credence might be. Instead, we say that in general, the larger the $k$ the better, i.e. $f'$ is better than $f$. The ideal case is when a functional is not fanatical for any $k \in (0, 0.5)$.

\begin{definition} \label{def:fanaticism} [Fanaticism]
A social welfare functional $f$ is said to be \emph{$k$-fanatical} (for some $k \in (0, 0.5)$) if for any ethical framework $F_y=(T_y, c_y)$ (where `$y$' denotes  `yielding') and any set of actions $A$, there exists an $F = (T, c)$ obtained by extending $F_y$ with $T = T_d \bigcup T_y$, where  $T_d \ne \emptyset$ %(`$d$' denotes `dominating') 
 is a dominant subset given $F$ and $A$, and:\\[3pt]
\noindent i) $k'$ is the total credence of the dominant theories $T_d$, i.e. $k' = \sum_{t \in T_d}c(t)$ and ii) $k' \le k$.
\end{definition}

In other words, $f$ is $k$-fanatical if for \textit{any} set of ethical theories $T_y$ and actions $A$ there is a set of ethical theories $T_d$ such that $T_d$ is a dominant subset in the framework $(T_y,c)$ extended with $T_d$, and $T_d$'s  total credence is at most $k$.

Finally,  a special case of fanaticism is when an \textit{swf} is fanatical for \textit{any} $k$, no matter how small $k$ is. These most extreme cases of fanaticism are called \emph{Pascalian} (after Pascal's Wager) in moral uncertainty \cite{sep-pascal-wager,tarsney2018moral}. %TODO comment on why this is especially bad?

\begin{definition} \label{def:Pascalian} [Pascalian fanaticism]
A social welfare functional $f$ is \emph{Pascalian} if it is $k$-fanatical for any $k \in (0, 0.5)$.
\end{definition}

\subsection{Formal Results} \label{sec:fanaticism:thems}
We now state formal results concerning the extent to which the \textit{swf}s studied in this paper are fanatical.

Note that proofs of all results in this section are included in the appendix%\footnote{See the appendix in [TODO arxiv link] filler text ... filler text ... filler}
. The general idea behind the different proofs is to find a property of the ethical theories, such that if it is sufficiently large, it allows a theory to dominate. For example, for $\mathit{mec}$ this is the minimum difference between any two action's evaluations; if this difference is sufficiently high, the theory will dominate all the others. For the non-fanaticism of $\mathit{hm}$ we show that such a property cannot exist.

Firstly, recall (Section \ref{sec:example}) that  FROBO's expected choiceworthiness calculations are dominated by deontology $d$ despite its low credence ($c(d) = 0.01$), and so FROBO chooses to enter the right room. In other words, MEC is vulnerable to fanaticism. In fact, MEC is Pascalian:
\begin{theorem}
The social welfare functional $\mathit{mec}$ (MEC) is Pascalian.
\end{theorem}
In Example \ref{ex:mm}, FROBO's minimum evaluations are dominated by deontology $d$; using Maximin, FROBO chooses to enter the right room. Indeed, Maximin is also Pascalian:

\begin{theorem}
The social welfare functional $\mathit{mm}$ (Maximin) is Pascalian.
\end{theorem}
Recall Example \ref{ex:kthm}. By trimming, FROBO can disregard extreme, low-credence evaluations; FROBO avoids fanaticism and chooses to enter the left room. In general, trimming enables partial avoidance of fanaticism, in the sense that:
\begin{theorem}
The social welfare functional $k$-$\mathit{thm}$ (k-trimmed Highest Mean) is not $k'$-fanatical for any $k' \le k$, but  is $k^*$-fanatical for any $k^* > k$.
\end{theorem}
Finally, Example \ref{ex:hm} illustrates that FROBO's use of the weighted median prevents FROBO from choosing the right room. Indeed, the median is completely non-fanatical, which we  formally state as follows:
\begin{theorem}
The social welfare functional $\mathit{hm}$ (Highest Median) is not $k$-fanatical for any $k \in (0, 0.5)$.
\end{theorem}

\section{Conclusion}
In this work we have defined fanaticism as a property of \textit{swf}s. We proved that MEC (and indeed Maximin) are vulnerable to an extreme Pascalian form of fanaticism.

Our paper thus presents a critique of MEC. In particular, we have shown that less fanatical modifications of MEC -- either weighted $k$-trimmed Highest Mean or weighted Highest Median -- are more appropriate, and where the latter, as opposed to the former, completely avoids fanaticism.  However, are either of these two `better' and what value of $k$ ought to be used for $k$-trimmed Highest Mean? In the moral uncertainty literature, there is a notion of \emph{stakes sensitivity}, \cite{DBLP:conf/icml/EcoffetL21,newberry2021parliamentary}, i.e. that \textit{swf}s ought to be sensitive to the magnitude of evaluations of ethical theories. Now, fanaticism is an oversensitivity to stakes and undersensitivity to credences; fanaticism is an extreme risk aversion to stakes but not credences. Therefore, the less fanatical an \textit{swf}, the less sensitive it is to stakes and the more sensitive it is to credences. For example, if there is a theory with more than $0.5$ credence, the Highest Median will ignore all other theories, no matter how large the stakes they posit. This may seem a steep price to avoid fanaticism; however, there is a noted a tension between stakes and credence sensitivity \cite{beckstead2021paradox,newberry2021parliamentary}. Thus, we ought, arguably, to find a compromise between these two, in which case the `best' solution is likely to be $k$-trimmed Highest Median for some moderate value of $k$. This would then allow an appropriate (i.e. not under- or over-) sensitivity to stakes \textit{and} credences; exploring this is future work.

While the motivation for our work focuses on an agent acting on behalf of a society in which each individual advocates for a particular ethical theory, it may well be that an individual agent may also adopt multiple ethical perspectives, with different credences. Indeed, the problem of fanaticism was originally defined with respect to individual agents \cite{ross2006rejecting, macaskill2014normative}. Our results equally apply in these scenarios, and would be especially relevant in scenarios where AI agents learn the ethically informed preferences of individual human agents (recall Footnote \ref{foot:cirl}).

Finally, as noted before, our work is not the first to import ideas from moral uncertainty into machine ethics, e.g. consider \cite{DBLP:journals/mima/Bogosian17,DBLP:conf/icml/EcoffetL21}, which uses a Reinforcement Learning approach to evaluate actions under moral uncertainty. Our work provides formal results that can support such applied contexts.
%a novel framework for studying aggregation of action evaluations under moral uncertainty, and the properties of these aggregations, such that principled evaluations can be used in such applied contexts.

\section*{Acknowledgements} 
This work was supported by UK Research and Innovation [grant number EP/S023356/1], in the UKRI Centre for Doctoral Training in Safe and Trusted Artificial Intelligence (www.safeandtrustedai.org). This work was also suported by VAE-VADEM TED2021-131295B-C32, funded by MCIN/AEI/10.13039/501100011033 and the European Union NextGenerationEU/PRTR.

\bibliography{mybib}

\begin{thebibliography}{28}
\providecommand{\natexlab}[1]{#1}

\bibitem[{Alexander and Moore(2021)}]{sep-ethics-deontological}
Alexander, L.; and Moore, M. 2021.
\newblock {Deontological Ethics}.
\newblock In Zalta, E.~N., ed., \emph{The {Stanford} Encyclopedia of
  Philosophy}. Metaphysics Research Lab, Stanford University, {W}inter 2021
  edition.

\bibitem[{Anderson, Anderson, and Armen(2005)}]{anderson2005towards}
Anderson, M.; Anderson, S.; and Armen, C. 2005.
\newblock Towards machine ethics: Implementing two action-based ethical
  theories.
\newblock In \emph{Proceedings of the AAAI 2005 fall symposium on machine
  ethics}, 1--7.

\bibitem[{Assembly et~al.(1948)}]{assembly1948universal}
Assembly, U.~G.; et~al. 1948.
\newblock Universal declaration of human rights.
\newblock \emph{UN General Assembly}, 302(2): 14--25.

\bibitem[{Balinski and Laraki(2007)}]{balinski2007theory}
Balinski, M.; and Laraki, R. 2007.
\newblock A theory of measuring, electing, and ranking.
\newblock \emph{Proceedings of the National Academy of Sciences}, 104(21):
  8720--8725.

\bibitem[{Beckstead and Thomas(2021)}]{beckstead2021paradox}
Beckstead, N.; and Thomas, T. 2021.
\newblock A paradox for tiny probabilities and enormous values.
\newblock \emph{No{\^u}s}.

\bibitem[{Bhargava and Kim(2017)}]{bhargava2017autonomous}
Bhargava, V.; and Kim, T.~W. 2017.
\newblock Autonomous vehicles and moral uncertainty.
\newblock \emph{Robot ethics}, 2.

\bibitem[{Bogosian(2017)}]{DBLP:journals/mima/Bogosian17}
Bogosian, K. 2017.
\newblock Implementation of Moral Uncertainty in Intelligent Machines.
\newblock \emph{Minds Mach.}, 27(4): 591--608.

\bibitem[{Dobbe, Gilbert, and Mintz(2020)}]{DBLP:conf/aies/DobbeGM20}
Dobbe, R.; Gilbert, T.~K.; and Mintz, Y. 2020.
\newblock Hard Choices in Artificial Intelligence: Addressing Normative
  Uncertainty through Sociotechnical Commitments.
\newblock In Markham, A.~N.; Powles, J.; Walsh, T.; and Washington, A.~L.,
  eds., \emph{{AIES} '20: {AAAI/ACM} Conference on AI, Ethics, and Society, New
  York, NY, USA, February 7-8, 2020}, 242. {ACM}.

\bibitem[{Ecoffet and Lehman(2021)}]{DBLP:conf/icml/EcoffetL21}
Ecoffet, A.; and Lehman, J. 2021.
\newblock Reinforcement Learning Under Moral Uncertainty.
\newblock In Meila, M.; and Zhang, T., eds., \emph{Proceedings of the 38th
  International Conference on Machine Learning, {ICML} 2021, 18-24 July 2021,
  Virtual Event}, volume 139 of \emph{Proceedings of Machine Learning
  Research}, 2926--2936. {PMLR}.

\bibitem[{Fabre(2021)}]{fabre2021tie}
Fabre, A. 2021.
\newblock Tie-breaking the highest median: alternatives to the majority
  judgment.
\newblock \emph{Social Choice and Welfare}, 56(1): 101--124.

\bibitem[{Gabriel(2020)}]{gabriel2020artificial}
Gabriel, I. 2020.
\newblock Artificial intelligence, values, and alignment.
\newblock \emph{Minds and machines}, 30(3): 411--437.

\bibitem[{Hadfield-Menell et~al.(2016)Hadfield-Menell, Dragan, P.Abbeel, and
  Russell}]{CIRL}
Hadfield-Menell, D.; Dragan, A.; P.Abbeel; and Russell, S. 2016.
\newblock Cooperative inverse reinforcement learning.
\newblock In \emph{NIPS'16: Proc. 30th Int. Conference on Neural Information
  Processing Systems}, 3916–3924.

\bibitem[{Haidt(2012)}]{haidt2012righteous}
Haidt, J. 2012.
\newblock \emph{The righteous mind: Why good people are divided by politics and
  religion}.
\newblock Vintage.

\bibitem[{Harsanyi(1955)}]{harsanyi1955cardinal}
Harsanyi, J.~C. 1955.
\newblock Cardinal welfare, individualistic ethics, and interpersonal
  comparisons of utility.
\newblock \emph{Journal of political economy}, 63(4): 309--321.

\bibitem[{Hurley and Lior(2002)}]{hurley2002combining}
Hurley, W.; and Lior, D. 2002.
\newblock Combining expert judgment: On the performance of trimmed mean vote
  aggregation procedures in the presence of strategic voting.
\newblock \emph{European Journal of Operational Research}, 140(1): 142--147.

\bibitem[{Hájek(2022)}]{sep-pascal-wager}
Hájek, A. 2022.
\newblock {Pascal’s Wager}.
\newblock In Zalta, E.~N.; and Nodelman, U., eds., \emph{The {Stanford}
  Encyclopedia of Philosophy}. Metaphysics Research Lab, Stanford University,
  {W}inter 2022 edition.

\bibitem[{Kernohan(2021)}]{kernohan2021descriptive}
Kernohan, A. 2021.
\newblock Descriptive Uncertainty and Maximizing Expected Choice-Worthiness.
\newblock \emph{Ethical Theory and Moral Practice}, 24(1): 197--211.

\bibitem[{List(2022)}]{sep-social-choice}
List, C. 2022.
\newblock {Social Choice Theory}.
\newblock In Zalta, E.~N.; and Nodelman, U., eds., \emph{The {Stanford}
  Encyclopedia of Philosophy}. Metaphysics Research Lab, Stanford University,
  {W}inter 2022 edition.

\bibitem[{MacAskill(2014)}]{macaskill2014normative}
MacAskill, W. 2014.
\newblock \emph{Normative uncertainty}.
\newblock Ph.D. thesis, University of Oxford.

\bibitem[{MacAskill(2022)}]{macaskill2022we}
MacAskill, W. 2022.
\newblock \emph{What we owe the future}.
\newblock Basic books.

\bibitem[{MacAskill, Bykvist, and Ord(2020)}]{macaskill2020moral}
MacAskill, W.; Bykvist, K.; and Ord, T. 2020.
\newblock \emph{Moral uncertainty}.
\newblock Oxford University Press.

\bibitem[{Maronna et~al.(2006)Maronna, Martin, Yohai, and
  Salibi{\'a}n-Barrera}]{maronna2006robust}
Maronna, R.; Martin, R.~D.; Yohai, V.; and Salibi{\'a}n-Barrera, M. 2006.
\newblock Robust statistics: Theory and practice.

\bibitem[{Martinho, Kroesen, and Chorus(2021)}]{martinho2021computer}
Martinho, A.; Kroesen, M.; and Chorus, C. 2021.
\newblock Computer Says I Don’t Know: An Empirical Approach to Capture Moral
  Uncertainty in Artificial Intelligence.
\newblock \emph{Minds and Machines}, 31(2): 215--237.

\bibitem[{Moor(2006)}]{moor2006nature}
Moor, J.~H. 2006.
\newblock The nature, importance, and difficulty of machine ethics.
\newblock \emph{IEEE intelligent systems}, 21(4): 18--21.

\bibitem[{Newberry and Ord(2021)}]{newberry2021parliamentary}
Newberry, T.; and Ord, T. 2021.
\newblock The parliamentary approach to moral uncertainty.
\newblock \emph{Future of Humanity}.

\bibitem[{Ross(2006)}]{ross2006rejecting}
Ross, J. 2006.
\newblock Rejecting ethical deflationism.
\newblock \emph{Ethics}, 116(4): 742--768.

\bibitem[{Sinnott-Armstrong(2022)}]{sep-consequentialism}
Sinnott-Armstrong, W. 2022.
\newblock {Consequentialism}.
\newblock In Zalta, E.~N.; and Nodelman, U., eds., \emph{The {Stanford}
  Encyclopedia of Philosophy}. Metaphysics Research Lab, Stanford University,
  {W}inter 2022 edition.

\bibitem[{Tarsney(2018)}]{tarsney2018moral}
Tarsney, C. 2018.
\newblock Moral uncertainty for deontologists.
\newblock \emph{Ethical Theory and Moral Practice}, 21(3): 505--520.

\end{thebibliography}

\appendix
\newpage%TODO include one more pagebreak for submission?
\section{Proofs of Formal results}
\begin{apptheorem} \label{theorem:mec}
The social welfare functional $\mathit{mec}$ (MEC) is Pascalian.
\end{apptheorem}
\begin{proof}
We prove that $\mathit{mec}$ is Pascalian by showing that $\mathit{mec}$ is $k$-fanatical for any $k \in (0, 0.5)$ (Definition \ref{def:Pascalian}). Let $k \in (0, 0.5)$ be arbitrary,  $F_y = (T_y, c_y)$ an  ethical framework and $A$ a set of actions. By  definition of fanaticism (Definition \ref{def:fanaticism}) $\mathit{mec}$ is shown to be $k$-fanatical if we can show that there exists a set of ethical theories $T_d$ such that:
1) $T_d$ is a dominant subset given $F$ and $A$, and    2)  $T_d$ has at most $k$ credence, i.e. $\sum_{t \in T_d}c(t) \le k$,
 where $F = (T,c)$ is  obtained by extending  $F_y = (T_y,c_y)$ with $T_d$.

We thus show that there exists a `fanatical' ethical theory $\mathit{ft}$ such that $T_d = \{\mathit{ft}\}$ and $c(\mathit{ft}) = k$.  We can see that 2) holds since $\sum_{t \in T_d}c(t) = c(\mathit{ft}) = k$. It remains to show that 1) holds, i.e. that $T_d$ is a dominant subset given $F$ and $A$. By Definition \ref{def:dominant-subset} we have to show that (recall  Notation \ref{not:evaluation-aggr}):\\[3pt]  \textbf{i)} $\mathit{mec}(F_d, A) \ne \mathit{mec}(F_y, A)$    \textbf{ii)} $\mathit{mec}(F, A) = \mathit{mec}(F_d, A)$.\\[3pt]
We start by making an observation with respect to i). Let $a^* \in A$ be an action that maximises $\mathit{wam}$ with respect to $F_y$, i.e. such that $a^*$ maximises $\mathit{wam}(F_y, a)$ for any $a \in A$. Then, by definition of $\mathit{mec}$ (Definition \ref{def:mec}), $a^*$ is a maximal element of the total order $\mathit{mec}(F_y, A)$. Therefore, if $a^*$ is not a maximal element of the total order $\mathit{mec}(F_d, A)$, then $\mathit{mec}(F_d, A) \ne \mathit{mec}(F_y, A)$ holds.

We now give a `partial' definition of $\mathit{ft}$ and prove i) using the above observation. Let $(a_1, ..., a_n)$ be any permutation of $A$ such that:\\[-13pt] $$a_n = b \neq a^* \; where \ ; n = |A|.$$\\[-13pt] Let $m > 0$ be a constant (whose precise value we later give). We define $\mathit{ft}$ for all $i \in [1, n]$:
\begin{equation}\label{eq:mec:ft}
ft(a_i) = (1/c(ft)) \times i \times m
\end{equation}
Informally, the term $(1/c(ft))$ allows $\mathit{ft}$ to `ignore' its credence. The term $i$ ensures that $a_i \prec_\mathit{ft} a_{i+1}$ holds. Finally, $m$ allows, as we will see, $\mathit{ft}$ to `ignore' the other theories.

Note that for any $i \in [1, n - 1]$ it is the case that $\mathit{ft}(a_i) < \mathit{ft}(a_{i+ 1})$ because $(1/c_y(ft)) \times i \times m < (1/c_y(ft)) \times (i + 1) \times m$, as $0 < 1/c_y(ft)$ and $0 <m$. Therefore, $\mathit{ft}$ orders the actions  $a_1 \prec_{ft} a_2 \prec_{ft} ... \prec_{ft} a_n = b$. Since $T_d$ is a singleton consisting of $\mathit{ft}$ it must be  that $\mathit{mec}(F_d, A) = \preceq_\mathit{ft}$. Therefore, since the only maximal element of $\preceq_\mathit{ft}$ is $b \ne a^*$, it must be  that $\mathit{mec}(F_d, A) \ne \mathit{mec}(F_y, A)$ holds. Thus if $\mathit{ft}$ is defined via Equation \ref{eq:mec:ft} (such that $m > 0$) then i) is shown.

We now prove ii) (i.e., $\mathit{mec}(F, A) = \mathit{mec}(F_d, A)$). In doing so, we give $m > 0$ a precise value.
First, we show that
\begin{equation}\label{eq:mec:1}
\mathit{wam}(F, a_i) = \mathit{wam}(F_y, a_i) + i \times m
\end{equation}
 
 By definition of $\mathit{wam}$ (see in text, prior to Definition \ref{def:mec}), the \textit{lhs} of Eq.\ref{eq:mec:1} expands to:
\begin{equation}\label{eq:mec:1:expand1}
\mathit{wam}(F, a_i) = \sum_{t \in T}c(t)t(a_i)
\end{equation}
Since $T = T_y \bigcup T_d$ and $T_d = \{ \mathit{ft} \}$ the \textit{rhs} of Eq. \ref{eq:mec:1:expand1} expands to:
\begin{equation}\label{eq:mec:1:expand2}
\sum_{t \in T}c(t)t(a_i) = \sum_{t \in T_y}c(t)t(a_i) + (c(ft) \times ft(a_i)) 
\end{equation}
By definition of $\mathit{wam}(F_y, a_i)$  we can substitute in the \textit{rhs} of Eq.\ref{eq:mec:1:expand2}, obtaining:
\begin{equation} \label{eq:mec:1:expand3}
\sum_{t \in T}c(t)t(a_i) =   \mathit{wam}(F_y, a_i) + c(ft) \times \mathit{ft}(a_i)
\end{equation} \label{eq:mec:1:expand4}
By definition, $\mathit{ft}(a_i) = (1/c(ft)) \times i \times m$, we cancel out $c(\mathit{ft})$ in the \textit{rhs} of Eq. \ref{eq:mec:1:expand3}, which is therefore equivalent to:
\begin{equation}
%\mathit{wam}(F_, a_i) + c(ft) \times \mathit{ft}(a_i) = 
\mathit{wam}(F_y, a_i) + i \times m
\end{equation}
We have thus shown that the \textit{lhs} and \textit{rhs} of Eq. \ref{eq:mec:1} are equal.\smallskip

Two total orders on $A$ are equal if  any two actions $a_i, a_j$ are ordered the same. Consider $\mathit{mec}(F, A) = \preceq_F$ and $\mathit{mec}(F_d, A) = \preceq_{F_d}$.
By Definition \ref{def:mec}:\\ $a_i \preceq_F a_j$ iff $\mathit{wam}(F, a_i) \le \mathit{wam}(F, a_j)$ and  $a_i \preceq_{F_d} a_j$ iff $\mathit{wam}(F_d, a_i) \le \mathit{wam}(F_d, a_j)$. \\Therefore, to show  ii) ($\mathit{mec}(F, A) = \mathit{mec}(F_d, A)$) it suffices to   show:\\[-12pt]
\begin{equation}\label{eq:mec:2}
\begin{split}
\mathit{wam}(F, a_i) &\le \mathit{wam}(F, a_j) \\ \Leftrightarrow\mathit{wam}(F_d, a_i) &\le \mathit{wam}(F_d, a_j)
\end{split}
\end{equation} 

\noindent By Eq.\ref{eq:mec:1} we obtain:
\begin{equation} \label{eq:mec:2:expand1}
\begin{split}
\mathit{wam}(F, a_i) &\le \mathit{wam}(F, a_j)\\ \Leftrightarrow \mathit{wam}(F_y, a_i) + i \times m &\le \mathit{wam}(F_y, a_j) + j \times m
\end{split}
\end{equation}

\noindent Rearranging obtains:
\begin{equation} \label{eq:mec:2:expand2}
\begin{split}
\mathit{wam}(F, a_i) &\le \mathit{wam}(F, a_j)\\ \Leftrightarrow \mathit{wam}(F_y, a_i) &\le \mathit{wam}(F_y, a_j) + (j - i) m
\end{split}
\end{equation}

\noindent Substituting  the \textit{rhs} of Eq. \ref{eq:mec:2:expand2} into the \textit{lhs} of  Eq. \ref{eq:mec:2} obtains:
\begin{equation} \label{eq:mec:2:expand3}
\begin{split}
\mathit{wam}(F_y, a_i) &\le \mathit{wam}(F_y, a_j) + (j - i) m\\ \Leftrightarrow \mathit{wam}(F_d, a_i) &\le \mathit{wam}(F_d, a_j)
\end{split}
\end{equation}
Note that since $T_d = \{\mathit{ft}\}$ and $c_d(\mathit{ft}) = 1$\footnote{Since $T_d$ consists only of a single theory $ft$, it must be that $c_d(\mathit{ft})=1$, given that the credences of theories in any ethical framework always add up to 1.}, for any $a_l \in A$ it must be that $\mathit{wam}(F_d, a_l) = c_d(\mathit{ft}) \mathit{ft}(a_l) = \mathit{ft}(a_l)$. By definition of $\mathit{ft}$ we thus have, for any $a_l \in A$ (below, $c$ is the credence function in $F = (T,c)$): 
\begin{equation}\label{eq:mec:3}
\mathit{wam}(F_d, a_l) = (1/c(ft)) \times l \times m
\end{equation}
Given  Eq. \ref{eq:mec:3}, we can thus express the rhs -- $\mathit{wam}(F_d, a_i) \le \mathit{wam}(F_d, a_j)$ -- of Eq. \ref{eq:mec:2:expand3}  as $(1/c(ft)) \times i \times m \le (1/c(ft)) \times j \times m$. Hence,  further simplifying obtains:
\begin{equation}\label{eq:mec:3:expand1}
%\begin{split}
\mathit{wam}(F_d, a_i)  \le \mathit{wam}(F_d, a_j) \\ \Leftrightarrow 0 \le (j - i)m
%\end{split}
\end{equation}

Since $m > 0$, we can divide by $m$ and  rearrange to obtain:
\begin{equation}\label{eq:mec:3:expand2}
%\begin{split}
\mathit{wam}(F_d, a_i) \le \mathit{wam}(F_d, a_j) \\ \Leftrightarrow i \le j
%\end{split}
\end{equation}
So, substituting the \textit{rhs} of Eq. \ref{eq:mec:3:expand2} into the \textit{rhs} of Eq.   \ref{eq:mec:2:expand3}, then  to show ii), it suffices to  show:
\begin{equation} \label{eq:mec:2:expand4}
%\begin{split}
\mathit{wam}(F_y, a_i)  \le \mathit{wam}(F_y, a_j) + (j - i) m \\ \Leftrightarrow i \le j
%\end{split}
\end{equation}
We  now prove $\Leftarrow$ and $\Rightarrow$.\smallskip

\noindent $(\mathbf{\Leftarrow})$ 
Assume $i \le j$. Hence $j-i > 0$. Rearranging the \textit{lhs} of Eq. \ref{eq:mec:2:expand4}, it therefore suffices to show:
\begin{equation}\label{eq:mec:left}
(\mathit{wam}(F_y, a_i) - \mathit{wam}(F_y, a_j)) / (j - i) \le m
\end{equation}
We now return to what the value of $m$ should be. First, recall the `role' of $m$ is to ensure $\mathit{ft}$ `ignores' the other theories; more precisely, to ensure that $i \le j$ iff $\mathit{wam}(F, a_i) \le \mathit{wam}(F, a_j)$. If we can prove this, we can prove ii).

By assumption $i \le j$, and so $j - i \ge 1$ (since both $i$ and $j$ are integers). Therefore, $1/(j-i) \le 1$. Thus if 
\begin{equation}\label{eq:mec:left:1}
(\mathit{wam}(F_y, a_i) - \mathit{wam}(F_y, a_j)) \le m
\end{equation}
holds, so must Eq. \ref{eq:mec:left}. We therefore define $m$ such that no matter what $a_i$ and $a_j$ might be, Eq. \ref{eq:mec:left:1} holds. Consider the maximum evaluation of any action by $F_y$
$s^+ = \mathit{max}\{\mathit{wam}(F_y, a) | a \in A\}$ and the minimum evaluation $s^- = min\{\mathit{wam}(F_y, a) | a \in A\}$. Let $s = \mathit{max}\{|s^+|, |s^-|\}$, i.e. the absolute value of the evaluation with the largest absolute value. That is, $|s^+|, |s^-| \le s$. By definition of $s^+$ being a maximal element, $\mathit{wam}(F_y, a_i) \le s^+$ holds and thus $\mathit{wam}(F_y, a_i) \le |s^+| \le s$ is true. Similarly, by definition of $s^-$, $ \mathit{wam}(F_y, a_j) \ge s^-$ must hold. Therefore,$ -\mathit{wam}(F_y, a_j) \le -s^- \le |s^-| \le s$.
Hence:
\begin{equation}\label{eq:mec:left:2}
(\mathit{wam}(F_y, a_i) - \mathit{wam}(F_y, a_j)) \le 2s
\end{equation}
Therefore, byEq. \ref{eq:mec:left:1}, $2s \le m$. However, recall that our proof of i) required that $m > 0$. Since it may be that $s = 0$  (if $s^+ = s^- = 0$), we define $m = 2s + 1$. Since $2s + 1 > 2s$, Eq. \ref{eq:mec:left:2} is satisfied and we have shown $\Leftarrow$.\\

\noindent (\textbf{$\Rightarrow$}) Assume that $\mathit{wam}(F_y, a_i) \le \mathit{wam}(F_y, a_j) + (j - i) m$. We show $i \le j$. Proof is by contradiction.
 Assume $i > j$. Then, $j - i < 0$. Therefore, by rearranging and dividing by $j - i$, we can rewrite the \textit{lhs} of Eq. \ref{eq:mec:2:expand4} as:
\begin{equation}\label{eq:mec:right}
(\mathit{wam}(F_y, a_i) - \mathit{wam}(F_y, a_j)) / (j - i) \ge m
\end{equation}
which  rearranging  obtains:
\begin{equation}\label{eq:mec:right:1}
(\mathit{wam}(F_y, a_j) - \mathit{wam}(F_y, a_i)) / (i - j) \ge m
\end{equation}
Thus we arrive at a contradiction if we can show:
\begin{equation}\label{eq:mec:right:2}
(\mathit{wam}(F_y, a_j) - \mathit{wam}(F_y, a_i)) / (i - j) < m
\end{equation}
Since $i - j \ge 1$, then $1 / (i - j) \le 1$. Therefore, to show that Eq. \ref{eq:mec:right:2} holds, it suffices to show:
\begin{equation}\label{eq:mec:right:3}
(\mathit{wam}(F_y, a_j) - \mathit{wam}(F_y, a_i)) < m
\end{equation}
By definition of $s^+$ as a maximal element, $\mathit{wam}(F_y, a_j) \le s^+$ holds and thus $\mathit{wam}(F_y, a_j) \le |s^+| \le s$ is true. Similarly, by definition of $s^-$, $ \mathit{wam}(F_y, a_i) \ge s^-$ must hold. Therefore,$ -\mathit{wam}(F_y, a_i) \le -s^- \le |s^-| \le s$.

Thus, $(\mathit{wam}(F_y, a_j) - \mathit{wam}(F_y, a_i)) \le 2s < m$. And so we have shown $\Rightarrow$ since we have arrived at a contradiction. Finally, since we have now shown that i) and  ii) holds, we have shown that $mec$ is Pascalian.
\end{proof}
\begin{apptheorem}
The social welfare functional $\mathit{mm}$ (Maximin) is Pascalian.
\end{apptheorem}
\begin{proof}
We show that $\mathit{mm}$ is $k$-fanatical for any $k$. Let $k \in (0, 0.5)$ be arbitrary,  $F_y = (T_y, c_y)$  an ethical framework,  $A$  a set of actions.

By  definition of fanaticism (Definition \ref{def:fanaticism}) $\mathit{mec}$ is $k$-fanatical if there exists a set of ethical theories $T_d$ such that:
1) $T_d$ is a dominant subset given $F$ and $A$, and   2)  $T_d$ has at most $k$ credence, i.e. $\sum_{t \in T_d}c(t) \le k$,
 where $F = (T,c)$ is  obtained by extending  $F_y = (T_y,c_y)$ with $T_d$.

We thus show that there exists a `fanatical' ethical theory $\mathit{ft}$ such that $T_d = \{\mathit{ft}\}$ and $c(\mathit{ft}) = k$. We can see that 2) holds since $\sum_{t \in T_d}c(t) = c(\mathit{ft}) = k$. It remains to show that 1) holds, i.e. that $T_d$ is a dominant subset given $F$ and $A$. By Definition \ref{def:dominant-subset} we have to show that (recall  Notation \ref{not:evaluation-aggr}):\\[3pt]  \textbf{i)} $\mathit{mm}(F_d, A) \ne \mathit{mm}(F_y, A)$    \textbf{ii)} $\mathit{mm}(F, A) = \mathit{mm}(F_d, A)$.\\[3pt]
We start by making an observation with respect to i). Let $a^*$ be an action with minimal evaluation, i.e. there exists no $a \in A$ such that    $min_{t \in T_y}t(a) < min_{t \in T_y}t(a^*)$. Then, by definition of $\mathit{mm}$ (Definition \ref{def:maximin}), $a^*$ is a maximal element of the total order $\mathit{mm}(F_y, A)$. Therefore, if $a^*$ is not a maximal element of the total order $\mathit{mm}(F_d, A)$, then $\mathit{mm}(F_d, A) \ne \mathit{mm}(F_y, A)$ holds.

For any action $a \in A$, let $m_a$ be the minimum evaluation of action $a$ by any theory in the framework $F_y$, i.e. $m_a = \mathit{min}(\{t(a)|t \in T_y\})$. Note that by definition of $a^*$, it holds that for any $a \in A$: $m_{a^*} \le m_a$.   We can define a fanatical theory $\mathit{ft}$  such that
\begin{equation}\label{eq:ft-mm} 
 \mathit{ft}(a^*)  = m_{a^*} - 2   \textrm{ and } \forall a \neq a^* :\mathit{ft}(a)  = m_{a^*} - 1
\end{equation}
Hence, for any $a \in A$:
\begin{equation}\label{eq:mm:1}
{\mathit{ft}(a)} \le m_{a^*} - 1 < m_{a^*} \le m_a
\end{equation}

To show i) ($\mathit{mm}(F_d, A) \ne \mathit{mm}(F_y, A)$)  note that $a^* \prec_{\mathit{F_d}} a$ (for $a \neq a^*$) where $\preceq_{F_d} = \mathit{mm}(F_d, A)$ by definition of $\mathit{ft}$ (Eq. \ref{eq:ft-mm}). Thus $a^*$ is not a maximal element of $\mathit{mm}(F_d, A)$, and so it must be that $\mathit{mm}(F_d, A) \ne \mathit{mm}(F_y, A)$.

To show ii) ($\mathit{mm}(F, A) = \mathit{mm}(F_d, A)$): we can derive that $min_{t \in T}t(a) = \mathit{ft}(a)$ holds for any $a$. This is because $min_{t \in T}(\{t(a) | t \in T_y\bigcup\{\mathit{ft}\}\} = ft(a)$, by  Eq. \ref{eq:mm:1}. Since maximin $\mathit{mm}$ orders actions based on which one has the higher minimum evaluation and the minimum evaluation is always given by $\mathit{ft(a)}$, it must be the case that $\mathit{mm}(F, A) = \mathit{mm}(F_d, A)$.

Since we have shown i) and ii), we have shown that $\mathit{mm}$ is Pascalian.
\end{proof}
\begin{apptheorem}
The social welfare functional $k$-$\mathit{thm}$ (k-trimmed Highest Mean) is $k'$-fanatical for any $k < k' < 0.5$.
\end{apptheorem}
\begin{proof}
The intuition behind this proof is that when $k' > k$, the fanatical theories are not `trimmed out' and thus $\textit{k-thm}$ (more-or-less) reduces to $\mathit{mec}$, which is $k'$-fanatical.

Let $k' \in (k, 0.5)$ be arbitrary,  $F_y = (T_y, c_y)$ an  ethical framework and $A$ a set of actions. By  definition of fanaticism (Definition \ref{def:fanaticism}) $k$-$\mathit{thm}$ is shown to be $k'$-fanatical if we can prove that there exists a set of ethical theories $T_d$ such that:
1) $T_d$ is a dominant subset given $F$ and $A$, and    2)  $T_d$ has at most $k'$ credence, i.e. $\sum_{t \in T_d}c(t) \le k'$,
 where $F = (T,c)$ is  obtained by extending  $F_y = (T_y,c_y)$ with $T_d$.

We thus show that there exists a `fanatical' ethical theory $\mathit{ft}$ such that $T_d = \{\mathit{ft}\}$ and $c(\mathit{ft}) = k'$.  We can see that 2) holds since $\sum_{t \in T_d}c(t) = c(\mathit{ft}) = k'$. It remains to show that 1) holds, i.e. that $T_d$ is a dominant subset given $F$ and $A$.

Let $F_x = (T_x, c_x)$ be an ethical framework where $T_x = \{t_x \}$ and $c_x(t_x) = 1$ such that for any $a \in A$, let
\begin{equation} \label{eq:kthm:x}
t_x(a) = \mathit{wam}(F^y_a, a) / (1 - k')
\end{equation}
where $F^y_a = (T_y, c_a)$ such that for $t \in (\mathit{bottom_{k'}}(a) \bigcup \mathit{top_{k'}}(a))$, $c_a(t) = 0$ and otherwise $c_a(t) = c_y(t)$.
In other words, $t_x$ is an ethical theory, whose evaluation of any action is proportional to the k-trimmed weighted arithmetic mean (see Definition \ref{def:kthm}); in later parts of this proof, it will be clear why we divide by $(1 - k')$.

%As a result of this construction of $t_x$ we can see that $a \preceq_\textit{k-thm} b$ iff $t_x(a) \le t_x(b)$. But note that since $T_x = \{t_x\}$ is a singleton, for all $a \in A$, $\mathit{wam}(F_x, a) = t_x(a)$. Consequently, 
%\begin{equation}\label{eq:kthm:1}a \preceq_{F_y} b \Leftrightarrow \mathit{wam}(F_x, a) \le \mathit{wam}(F_x, b)\end{equation}
%where $\preceq_{F_y} = \textit{k-thm}(F_y, c)$.

Let $F' = (T', c')$ be the framework obtained by extending $F_x$ with $T'$ where $T' = T_x \bigcup T_d = \{t_x, \mathit{ft}\}$, such that $c'(\mathit{ft}) = k'$ and $c'(t_x) = 1 - k'$. We will now prove that 
\begin{equation}\label{eq:kthm:2}
a \preceq_{F'} b \Leftrightarrow \mathit{wam}(F', a) \le \mathit{wam}(F', b)
\end{equation}
where $\preceq_{F'} = \textit{k-thm}(F, c)$.

By definition of $\textit{k-thm}$ (Definition \ref{def:kthm}), $a \preceq_{F'} b$ iff $\mathit{wam}(F_a, a) \le \mathit{wam}(F_b, b)$, where\\
\noindent $F_a = (T, c_a)$, $c_a(t) = 0$ for $t \in (\mathit{bottom_k}(a) \bigcup \mathit{top_k}(a))$, else $c_a(t) = c(t)$;\\
\noindent $F_b = (T, c_b)$, $c_b(t) = 0$ for $t \in (\mathit{bottom_k}(b) \bigcup \mathit{top_k}(b))$, else $c_b(t) = c(t)$.

Note that because $c(\mathit{ft}) = k' > k$, $\mathit{ft}$ can not be `trimmed out', i.e. for any $a'\in A$, $\mathit{ft} \notin (\mathit{bottom_{k'}}(a') \bigcup \mathit{top_{k'}}(a'))$ . Thus, for any $a' \in A$, $c_{a'}(\mathit{ft}) = c(\mathit{ft}) = k'$.

By the above and by the definition of $\mathit{wam}$, for any $a' \in A$, 
\begin{equation}\label{eq:kthm:3}
\begin{split}
\mathit{wam}(F_{a'}, a') &= \sum_{t \in T} c_{a'}(t)t(a') \\ &= k'\mathit{ft}(a') + \sum_{t \in T_y} c_{a'}(t)t(a')
\end{split}
\end{equation}
(recall that $T_y = T \setminus\{\mathit{ft}\}$). Note that by definition of $\textit{k-thm}$ (Definition \ref{def:kthm})$\sum_{t \in T_y} c_{a'}(t)t(a') = \mathit{wam}(F^y_{a'}, a')$. By Eq. \ref{eq:kthm:x}, we thus have $\sum_{t \in T_y} c_{a'}(t)t(a') = (1-k')t_x(a')$. Therefore, overall
\begin{equation}\label{eq:kthm:4}
\mathit{wam}(F_{a'}, a') = k'\mathit{ft}(a') + (1-k')t_x(a')
\end{equation}
Note that the \textit{rhs}, by definition of $\mathit{wam}$, is thus equal to $\mathit{wam}(F', a'))$. Hence,
\begin{equation} \label{eq:kthm:5}
\mathit{wam}(F_{a'}, a') = \mathit{wam}(F', a')
\end{equation}
By Eq. \ref{eq:kthm:5} and by the definition of  $\textit{k-thm}$, we have thus proven Eq. \ref{eq:kthm:2}.

Let $\preceq_\mathit{mec} = \mathit{mec}(F', a)$. Then by definition (Definition \ref{def:mec})
\begin{equation}\label{eq:kthm:6}
a \preceq_\mathit{mec} b \Leftrightarrow \mathit{wam}(F', a) \le \mathit{wam}(F', b)
\end{equation}
From Eq. \ref{eq:kthm:2} and Eq. \ref{eq:kthm:6}, we can thus derive that $\preceq_\mathit{mec} = \preceq_{F'}$.

From the proof of Theorem \ref{theorem:mec} we know that it's possible to define $\mathit{ft}$ such that $T_d = \{\mathit{ft}\}$ is a dominant subset for $\preceq_\mathit{mec}$. Since $\preceq_\mathit{mec} = \preceq_{F'}$, $T_d$ must be a dominant subset for $\preceq_{F'}$ too.

\end{proof}

\begin{apptheorem}
The social welfare functional $k$-$\mathit{thm}$ (k-trimmed Highest Mean) is not $k'$-fanatical for any $0 < k' \le k$.
\end{apptheorem}
\begin{proof}
The intuition behind this proof is that theories with less than $k$ credence are `trimmed out' if they have extreme values. As such, fanaticism is avoided.
We prove that $\textit{k-thm}$ is not $k'$-fanatical for any $k\in(0, k)$, by giving a counter-example such that there cannot exist a dominant subset.

Let $A$ be a set of actions and $F_y = (T_y, c_y)$ an ethical framework. By the definition of fanaticism (Definition \ref{def:fanaticism}) if  there is no framework $F$ such that: $F$ is obtained by extending $F_y$ with $T = T_y \bigcup T_d$ and  $T_d$ is any subset of theories with at most $k$ credence, and $T_d$ is a dominant subset, then $\mathit{hm}$ is not $k$-fanatical.

Let $T_y = \{t\}$ and $A = \{a, b\}$ such that $c_y(t) = 1$, $t(a) = 1$ and $t(b) = 0$. Let $F$ be obtained by extending $F_y$ with $T = T_y \bigcup T_d$, where $T_d$ is any subset of theories with at most $k'$ credence, i.e. $\sum_{t' \in T_d}c(t') \le k'$. Given that credences in $T = \{t\} \bigcup T_d$ sum up to 1, we have that $c(t) \ge 1-k'$. Since $k' < k < 0.5$, we then have $c(t) > 0.5$.

We show that for any $t' \in T_d$ and any $a' \in A$ it must be the case that either $t' \in \mathit{bottom_k}(a')$ or $t' \in \mathit{top_k}(a')$.

First, let $a' \in A$. Because $c(t) > k$, $t \notin \mathit{bottom_k}(a')$ and $t \notin \mathit{top_k}(a')$; by their definition, either set contains theories of only up to total $k'$ credence.

Therefore, for any other theories $t' \in T_d$, $t' \in \mathit{bottom_k}(a')$ or $t' \in \mathit{top_k}(a')$. This is because the sum of the total credence of theories in $T_d$ is less than $k$, i.e. $\sum_{t \in T_d}c(t) = k' \le k$. And so both $k'end$ and $k'start$ must be so that $t$ is excluded but no other theory.

Now, since for all $t' \in T_d$ either $t' \in \mathit{bottom_k}(a')$ or $t' \in \mathit{top_k}(a')$ is true, it must be the case that they get `trimmed out' by \textit{k-thm} for any $a'$. Therefore, the theories in $T_d$ are ignored and hence cannot form a dominant subset.
\end{proof}

\begin{apptheorem}
The social welfare functional $\mathit{hm}$ is not $k$-fanatical for any $k \in (0, 0.5)$.
\end{apptheorem}
\begin{proof}
We prove that $\mathit{hm}$ is not $k$-fanatical for any $k\in(0, 0.5)$, by giving a counter-example such that there cannot exist a dominant subset. Let $A$ be a set of actions  and $F_y = (T_y, c_y)$ an ethical framework. By the definition of fanaticism (Definition \ref{def:fanaticism}) if  there is no framework $F$ such that: $F$ is obtained by extending $F_y$ with $T = T_y \bigcup T_d$ and  $T_d$ is any subset of theories with at most $k$ credence, and $T_d$ is a dominant subset, then $\mathit{hm}$ is not $k$-fanatical.

Let $T_y = \{t\}$ and $A = \{a, b\}$ such that $c_y(t) = 1$, $t(a) = 1$ and $t(b) = 0$. Let $F$ be obtained by extending $F_y$ with $T = T_y \bigcup T_d$, where $T_d$ is any subset of theories with at most $k$ credence, i.e. $\sum_{t' \in T_d}c(t') \le k$. Given that credences in $T = \{t\} \bigcup T_d$ sum up to 1, we have that $c(t) \ge 1-k$. Since $k < 0.5$ it must  be that $c(t) > 0.5$

We  now show that by the properties of the weighted median, and given that  $c(t) > 0.5$, then it must be that for any $a' \in A$: $\mathit{wmedian}(F, a') = t(a')$ . Let $se(F, a')$ be a sorted list of evaluations of $a'$ and $st(F, a')$ the corresponding sorted list of theories. Let $j$ be a number such that $st(F, a')_j = t$. The equations 
\begin{equation}\label{eq:hm:1}
\sum_{i \in [1, m]}c(st(F,a')) \le 1/2 \textrm{ and } \sum_{i \in [m+1,n]}c(st(F,a')) \le 1/2
\end{equation}
are satisfied by $m = j$. This is because $c(t) > 0.5$ and so for any subset of theories $T' \subseteq (T \setminus \{t\})$, it is the case that $\sum_{t' \in T'}c(t') < 0.5$.

Now either $m$ is uniquely determined or there are two such values for $m$. If $m$ is uniquely determined (and thus $m = j$) then by definition of $\mathit{wmedian}$ for any $a'$, we have that $\mathit{wmedian}(F, a') = st(F, a')_j = t(a')$.
Otherwise, if $m$ is not uniquely determined, we arrive at a contradiction. To see this, assume that there exists $l \ne j$ such that $l$ satisfies Equations \ref{eq:hm:1}. Note that either $j < l$ or $j > l$. If $j < l$, then
\begin{equation}
\sum_{i \in [1, l]}c(st(F,a')) = c(t) + \sum_{i \in [1, l], i \ne j}c(st(F,a')) > 0.5
\end{equation}
(given $t = st(F, a')_j$ and $c(t) > 0.5$). Therefore, $j < l$ cannot be the case. But symmetrically, if $j > l$ then
\begin{equation}
\sum_{i \in [l+1, n]}c(st(F,a')) = c(t) + \sum_{i \in [l+1, n], i \ne j}c(st(F,a')) > 0.5
\end{equation}
and so $j > l$ also cannot be the case. Therefore, $m$ must be uniquely determined and so for any $a' \in A$, $\mathit{wmedian(F, a')} = t(a')$.

Let $\preceq_F$ denote $\mathit{hm}(F, A)$. Note that $\mathit{wmedian}(F, b) = t(b) = 0$ and $\mathit{wmedian}(F, a) = t(a) = 1$. Hence, by definition of $\mathit{hm}$ (Definition \ref{def:hm}) it must be the case that $b \prec_F a$.

Let $\mathit{hm}(F_y, A)$ = $\preceq_{y}$. Trivially, the weighted median is such that $\mathit{wmedian}(F_y, b) = t(b) = 0$ and $\mathit{wmedian}(F_y, a) = t(a) = 1$. Thus, $b \prec_y a$ holds. Therefore,  $\preceq_F = \preceq_y$.

Since $\mathit{hm}(F_y, A) = \mathit{hm}(F, A)$, $T_d$ cannot be a fanatical subset. Therefore, for arbitrary $k \in (0, 0.5)$, $\mathit{hm}$ is not $k$-fanatical.
\end{proof}
\end{document}